\documentclass[dvipsnames,format=sigconf,nonacm]{acmart}

\usepackage{amsmath}
\usepackage{amsthm}
\usepackage{amscd}
\usepackage{amsfonts}
\usepackage{graphicx}
\usepackage{color}
\usepackage{algorithm}
\usepackage{algorithmicx}
\usepackage{algpseudocode}
\PassOptionsToPackage{colorlinks=true,linkcolor=blue,anchorcolor=blue,citecolor=blue,urlcolor=blue}{hyperref}

\usepackage[font=bf]{caption} 
\usepackage{subcaption}

\newtheorem*{remark}{Remark}

\newcommand{\ie}{{\em i.e.,} }
\newcommand{\eg}{{\em e.g.,} }

\newcommand{\Cs}{\mathcal C}

\newcommand{\D}{\mathcal D}
\newcommand{\Z}{\mathbb Z}
\newcommand{\Ss}{\mathcal S}

\newcommand{\ol}[1]{\overline{#1}}
\newcommand{\defequals}{\; \stackrel{\text{def}}{=} \;} 
\renewcommand{\phi}{ \emptyset }
\newcommand{\prob}[1]{\mbox{Prob}\left\{ #1 \right\}}

\author{Matthew Roughan}
\orcid{0000-0002-7882-7329}
\affiliation{
  \institution{School of Computer and Mathematical Sciences, University of Adelaide}
  \country{Australia}
}
\email{matthew.roughan@adelaide.edu.au}

\begin{document} 
\label{firstpage}

\def\lemmaautorefname{Lemma}
\def\algorithmautorefname{Algorithm} 
\def\sectionautorefname{Section \!\!}
\def\subsectionautorefname{Section \!\!}
\def\equationautorefname~#1\null{%
  (#1)\null
}

\setcounter{page}{1}

\title{Evolutionary Generation of Random Surreal Numbers for Benchmarking}

\begin{abstract}
There are many areas of scientific endeavour where large, complex datasets are needed for benchmarking. Evolutionary computing provides a means towards creating such sets. As a case study, we consider Conway's Surreal numbers. They have largely been treated as a theoretical construct, with little effort towards empirical study, at least in part because of the difficulty of working with all but the smallest numbers. To advance this status, we need efficient algorithms, and in order to develop such we need benchmark data sets of surreal numbers. In this paper, we present a method for generating ensembles of random surreal numbers to benchmark algorithms. The approach uses an evolutionary algorithm to create the benchmark datasets where we can analyse and control features of the resulting test sets. Ultimately, the process is designed to generate networks with defined properties, and we expect this to be useful for other types of network data. 
\end{abstract} 
 
\maketitle

\section{Introduction}
 
\newlength{\figurewidth}
\setlength{\figurewidth}{0.85\columnwidth} 

There are many areas of scientific endeavor where large, complex datasets are needed for benchmarking. Evolutionary computing has already provided a means towards creating these for correctness testing~\cite{jones96:_autom,pargas99:_test,adeniyi16:_improv_genet_algor_based_test}. We propose that a similar approach is useful for generating benchmarks to test performance. 

For instance, performance testing of network protocols and anomaly
detection algorithms have always required large, complex datasets for
benchmarking and acquiring these from real data has major challenges
around privacy, symmetry, and the large scale of the data
required~\cite{roughan2008}.  Another application needing large,
complex, synthetic data is medical research~\cite{giuffre13:_harnes},
where there is a growing number of commercial companies creating
synthetic data.

Here we apply the idea of using evolutionary computing to generate controlled, synthetic data for benchmarking algorithms in a theoretical context--the surreal numbers--where we have a simple theoretical problem, but complex, recursive calculations, whose computational complexity is difficult to analyse. 

Conway's surreal numbers~\cite{conway2000numbers,knuth1974surreal} are an unconventional construction for conventional numbers (including the naturals, rationals, reals and ordinals). However, almost all the literature on surreal numbers views them as a purely theoretical construction.
Missing from the literature are algorithms to efficiently compute results, for instance even standard numeric functions such as {\tt floor}.  Algorithms provide means to test new ideas on surreal numbers, but they must be effective. Hence, we need to have a benchmark set of surreal numbers. 
 
Typical related works to create tests for software aim to use genetic algorithms to optimise aspects such as code coverage~\cite{jones96:_autom,pargas99:_test,adeniyi16:_improv_genet_algor_based_test}. However, our goal is not just testing correctness, but also benchmarking alternative algorithms. Although there is a long history of using genetic algorithms to generate test sets the idea of generating performance tests seems much more recent~\cite{wilde20:_evolut}, however, even in \cite{wilde20:_evolut} they seek to find special `hard' or `easy' cases, rather than simply benchmark performance. For our benchmarks to be valid, we don't wish to steer the results to particular cases -- we want to test algorithms in the `wild.'

However, we also need test sets of surreal numbers to discover and
validate new hypotheses about the surreal numbers such as those
associated with {\em birthday
  arithmetic}~\cite{simons17:_meet_the_surreal_number,ROUGHAN_s2023}.
So here we approach the problem slightly differently from a
traditional genetic algorithm in that genetic material is structurally
used to create new surreals explicitly, \ie there is no indirection
through a representation, and no explicit fitness is used to guide the
algorithm.

Nevertheless, we must retain some control over the complexity of the dataset. It must be complicated enough to stress algorithms, but surreal algorithms are challengingly recursive, so the test set could easily be too complex. And the required complexity is likely to be different for different algorithms, so we need to be able to analyze and predict properties of the resulting data.

This paper presents a method using an evolutionary approach for
generating random surreal numbers that we can mathematically analyse
and control, and hence use to benchmark algorithms.  For instance, we
can derive the {\em generation} distribution for these ensembles,
which can be used as a measure of their complexity, but which is also
of interest
itself~\cite{simons17:_meet_the_surreal_number,ROUGHAN_s2023}.

One can usefully think of surreal numbers as Directed Acyclic Graphs (DAGs) (with a small set of additional properties) and as such the method presented here is also a method for generating an ensemble of random DAGs, with control over size and density. 

In summary, the contributions of this paper are: 
\begin{enumerate}
  
  \item A means to generate a random ensemble of surreal numbers with
    controlled complexity (describe in \autoref{alg:main}). We provide
    the code and a set of examples as an extension to the package
    \url{https://github.com/mroughan/SurrealNumbers.jl}. 
 
  \item An analysis of the resulting ensemble showing (i) a form of
    weak stationarity, (ii) a distribution of elements that have a
    controllable distribution of complexity, and (ii) and
    understanding of how integers arise in the distribution. 
  
  \item An apparently new univariate, discrete, two-parameter
    stochastic distribution that we christen the {\em surreal
      distribution}. 

\end{enumerate}

The results here are not peculiar to surreal numbers. Evolutionary computing such as shown here provides an ideal means towards generating network data. Such are needed for many uses, for instance, Waxman~\cite{b.m.waxman88:_graph}, created his eponymous model for random graphs as a side note in a paper on designing new network protocols. However, models such as his omit structure in favour of statistical similarity. The approach used here allows the creation of structural constraints--in this case that the graph is a DAG--while matching statistical properties. Evolutionary algorithms are ideal for incorporating constraints into the model-building process. 

\section{Preliminaries and Related Work} 

There are several good tutorials or books on surreal numbers, \eg
\cite{conway2000numbers,knuth1974surreal,tondering13:_surreal_number_introd,grimm12:_introd_surreal_number,simons17:_meet_the_surreal_number}. Here we will present a minimal description
in order to provide a clear context for this work. 

A surreal number $x$ is defined in terms of its left and right sets $X_L$ and $X_R$, which are sets of surreal numbers (or empty). A valid surreal number requires that there are no elements $x_R \in X_R$ that are $\leq$ any element $x_L \in X_L$. We start by defining $\bar{0} = \{ \phi \mid \phi \},$ and from there construct all other surreals recursively.

We denote surreal numbers in lower case, and sets in upper case, with the
convention that $X_L$ and $X_R$ are the left and right sets of $x$,
and we write a surreal {\em form} as 
\[ x \defequals \{ X_L \mid X_R \},\]
where no element of $X_L$ is $\geq$ any element of $X_R$. 

The literature on surreal
numbers interweaves the numbers with their {\em forms}.  This is best explained by an
analogy to rational numbers. We can think of any given rational number
in terms of its \emph{form}, \eg $p/q$, but $2 p / 2q = p / q$, \ie there
are two forms with the same \emph{value}. In fact, there are an
infinite number of rational number forms with any given value.  We
usually refer to these as the same ``number.''  But the interest of this paper is primarily the forms because that is what an algorithm must work with.

In mathematical terms, any set of surreal numbers is actually a {\em setoid}, \ie a set equipped with an equivalence relation (here equality in value). It is common to reduce setoids\footnote{In fact when we consider the class of surreals we are dealing with a {\em braided ring setoid} because the surreals are equipped with standard arithmetic operations that are preserved under the equivalence relations.} to their {\em quotient set}, \ie a set of elements that is unique under the equivalence, by collapsing the equivalence classes down to a single element. Here we wish to maintain the subtlety of equivalence versus identity.

Hence, will use 
Keddie's conventions\footnote{Conway defines similar terms~\cite{conway2000numbers}[p.15] but uses different notation.}~\cite{keddie94:_ordin_operat_surreal_number} and say
\begin{itemize}
    \item two surreal forms are {\em identical} if they have identical left and right sets, and we denote this by $==$; 

    \item two surreal forms are {\em equal} (equivalent) if they have the same value, and we denote this by $\equiv$; and
    
    \item we use $\defequals$ for definitions.
    
\end{itemize}
We reserve single equals signs for real numbers. 



\subsection{Dyadic numbers and canonical forms}
 
The left and right sets of a surreal form can be infinite, and this 
leads to many interesting facets of the surreal numbers (the ordinals,  
for instance), but here we will only consider surreal numbers with
finite representations, the class of which we denote $\Ss$. These are the so-called \emph{short} numbers \cite[Def
2.24]{schleicher2005introduction}. The restriction to short surreals may seem limiting, but we are only restricting ourselves to surreal forms that can be generated by a finite stochastic process. 

\begin{figure}[!t] 
  \centering
    \includegraphics[width=0.16\textwidth]{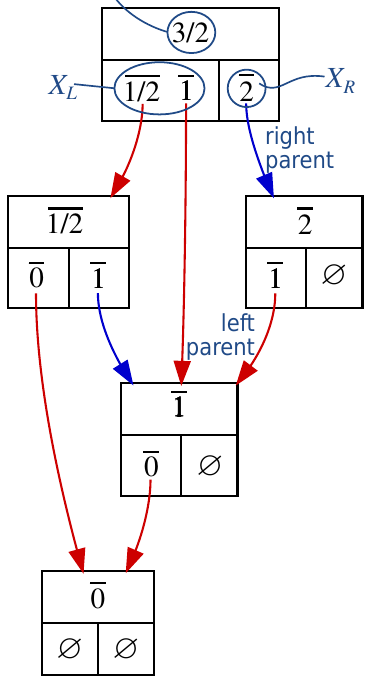}
  \caption{A DAG depicting a surreal form of 3/2. Boxes represent a
    surreal number (value in the top section, and the
    left and right sets shown in the bottom sections),
    with left and right parents shown by red and blue arrows.} 
    \label{fig:x0}
    \vspace*{-5mm} 
\end{figure}

These surreals are also an important and useful subset because they
correspond to the \emph{dyadic} numbers $\D$, which are rational
numbers of the form $n / 2^k$, where $n$ and $k$ are integers~\cite[p.27]{simons17:_meet_the_surreal_number}. 
The standard mapping from dyadic rational to surreal number is the Dali
function $d: \D \rightarrow \Ss$,~\cite{tondering13:_surreal_number_introd}. 
Defined
recursively, the Dali construction provides a simple and unique
\emph{canonical} form for each equivalence class of surreal forms. We denote canonical forms by putting a line above the number,
\eg
\[ \overline{1} \defequals d(1) \defequals \{
\overline{0} \mid \phi\} == \big\{ \{ \phi \mid \phi\} \mid \phi
\big\} 
.\]
Returning to the analogy of rational numbers, the canonical form is
similar to a rational number expressed in least terms, \ie the
form $p/q$ where $p$ and $q$ are integers without any (non-trivial) common factors. We denote the set of canonical forms by $\Cs \subset \Ss$. 

The Dali construction erects a scaffolding for surreals in terms of sets, but it is easier to visualise them as (labelled) Directed Acyclic Graphs (DAGs), \eg see~\autoref{fig:x0}. The figure shows a
surreal as a node in a connected graph with left and right
sets shown separately. The Dali construction creates only one surreal number form for each value, and these are the simplest, canonical forms, but there are an infinite number of others. For instance, \autoref{fig:x0} shows the DAG of a non-canonical form with value $3/2$.

Not all DAGs represent a surreal number form. Every surreal form's DAG has a single root (the node representing the number itself) and single sink-node, which is always the canonical origin $\ol{0}$.

\begin{definition}
    We refer to the elements of the left and right sets of a surreal $x$ as its left- and right-parents and we denote them by $P_L(x)$ and $P_R(x)$. We refer to the union of left and right sets as the parents and denote it $P(x) = P_L(x) \cup P_R(x)$. 
\end{definition}

Other authors use terms such as left and right {\em options}, from the
game-theory roots of surreal numbers~\cite{conway2000numbers}. The notion of parents is appealing, however, because it leads to a useful descriptor of a surreal form: its {\em generation} or {\em birthday}~\cite{knuth1974surreal}. 

\begin{definition}
    The {\em generation} $g(x)$ of a surreal $x$ is one more than that of the largest generation parent, \ie $g(\bar{0}) = 0,$ and
    \[ g(x) = 1 + \max_{p \in P(x)} g(p). \]
\end{definition}

\begin{remark}  The only surreal form with no parents is $\bar{0} 
  \defequals \{ \phi \mid \phi \}$ and this is also the only surreal form from generation 0.
  \end{remark}

We say a surreal $x$ is older than $y$ if it
comes from an earlier generation or has an earlier birthday, \ie $g(x)
< g(y).$ Older surreals are also called {\em simpler} by
Conway and others~\cite{conway2000numbers,simons17:_meet_the_surreal_number}. 

\begin{remark}  The rationale for equating lower generation with simplicity comes from several directions: it is true that the lowest generation form will have minimal representation size, but also the definition links to the defining the value of a surreal and other aspects of these numbers. Moreover, the generation determines the maximum depth of recursion required for computations on the surreal, and hence it is a key metric when considering the computational complexity of calculations using the surreal. Hence, this idea of simplicity is not an imposed ideal, it is a natural definition.
\end{remark}

\medskip

\section{The Synthesis Process}

To test algorithms we need to be able to set benchmark problems that are neither trivial nor impossibly hard. The standard Dali construction creates a so-called dyadic tree, but these are, by definition, the very simplest surreal forms. We need control over complexity. And although we could dream up any number of processes to create random surreals, we must be wary. Simple seeming approaches--for instance, use of arithmetic operations--can result in very, very complicated outputs~\cite{ROUGHAN2019293}.  
We show here how we can obtain that through an evolutionary algorithm.

Our process mimics a genetic algorithm in that it maintains a population of
surreal forms from which we select parents for a new population. The
algorithm also has some analogies with Markov Chain Monte Carlo
(MCMC) sampling methods. 

The term {\em generation} is already in use for surreal numbers, so we call each
iteration of a population a {\em clade}\footnote{We are somewhat abusing the
conventional definition of clade from biology, in which it would refer to a
group with a common ancestor -- here it is a group with a common
ancestral population.}.

The algorithm is detailed in \autoref{alg:main}, with its input parameters listed below: 

\begin{enumerate}
    \item The initial maximum generation $g^{(0)}_{max}$;

    \item The population size $n$, of each clade; 

    \item The number of iterations $m$, \ie the number of clades created; 

    \item The distribution $D_p$ of the number of parents selected for each new surreal;  

    \item A weighting function $w(g)$; and

    \item The distribution $D_s$ of the split point between the left and
      right parents. 
    
\end{enumerate}

\begin{algorithm}
\caption{Surreal number ensemble creation algorithm.}
\label{alg:main}
\begin{algorithmic}
\State Initialise the clade $\Cs_0$ with all canonical surreal forms with generation no more than $g_0^{max}$. 
\For{$i=1$ to $m$}
    \Comment{create a new clade $\Cs_i$}
    \For{$j=1$ to $n$}
       \State Generate a number of parents $n_p \sim D_p$
       \State Select $n_p$ surreals from $\mathcal{C}_{i-1}$ by the weighting $w(g(x))$.
       \State Sort the parents into an ordered list $P$
       \State Choose a split point $s \sim D_s(n_p)$
           \Comment{ $s \in [0,1,\ldots, n_p]$}
       \If{$s > 0$}
           \State $X_L = \{ P[1, \ldots, s] \}$ 
       \Else
           \State $X_L = \phi $ 
       \EndIf
       \If{$s \leq n_p$}
           \State $X_R = \{ P[s+1, \ldots, n_p] \}$ 
       \Else
           \State $X_R = \phi $ 
       \EndIf
       \State $x_{ij} \gets \{ X_L \mid X_R \}$
    \EndFor
\EndFor
\end{algorithmic}
\end{algorithm}

This is not strictly a genetic algorithm because (i) there is no `fitness' criteria being applied\footnote{One might think of the weight function as a fitness-derived selection criteria, but it is more mathematically tractable (as we show later) to use a simple weight and derive the ensemble properties than to approach this as fitness.}, and (ii) the parents here are not
being used for their `genetic material', they are parents in the
sense in which surreal numbers are constructed. They are therefore
present as subgraphs in the new surreal forms. However, it represents an evolutionary approach, building one clade from the previous to attain a controlled and stable level of complexity. 

We use the simplest possible choices for the distributions: 
\begin{enumerate}
    \item The number of parents needs to be a non-negative, univariate, discrete distribution, hence we use $D_p = Poisson(\lambda)$.

    \item Given $n_p = |P|$ parents, the split point distribution is over the
      integers from $0$ to $n_p$ and should be symmetric about
      $n_p/2$. Tested here are (i) $D_s(n_p) = U(0,n_p),$ (the Uniform
      discrete distribution) and (ii) $D_s = Bin(n_p,1/2),$ (the
      Binomial distribution on $n_p$ trials with probability 1/2).  
     
\end{enumerate}

We could stop the process at any iteration, and obtain a viable ensemble for use in testing, but it is more appealing to find processes that converge towards a stationary distribution.
However, it is not at all obvious that the process above converges. For instance, we note
that the generation of a surreal form is one more than that of its
youngest parent, and hence we might expect the maximum generation of
each clade to drift upwards. We show that with suitable weighting this force can be counteracted.

The initial population (of canonicals) is biased towards its largest generations because there are $2^k$ canonical surreal forms from generation $k$. Hence, in order to converge to a stable distribution we reduce this bias using the weighting function $w(\cdot)$ (as the generation distribution is discrete we will write $w_k$) to reduce the selection of larger generations, \ie we choose $w_k$ to be a decreasing function of $k$ (we will be more specific in the following section). 

With such a weighting function, we can show that this process has a type of weak-sense convergence to stationarity. That is, certain properties of the population converge. It is not the conventional second-order weak-sense convergence, but there is a valid analogy. We will explain convergence in the following section, but first we present some small implementation details. 
 
\section{The Surreal Distribution}

A surreal form's generation (or birthday) is a key indicator of
its complexity (see earlier discussion). Hence, a key goal here is to
control the evolution of the ensemble in such a way that (i) the process converges in some sense, and (ii) we understand the resulting generation distribution. 

Here we take $g_k$ to be the proportion of surreal forms in a clade $\mathcal{C}$ with generation $k$, and 
\( G_k = \sum_{i=0}^k g_i, \)
so that $g_k$ and $G_k$ are the empirical probability-mass function (PMF) and 
cumulative distribution function (CDF), respectively.

The generation of a surreal depends on those of its parents, but not on whether they are left or right parents, and so we can ignore the splitting function for the moment. 

The number of parents selected comes from the Poisson distribution with parameter $\lambda$ (the mean of the distribution). That is, the number of parents $n_p = |P|$ is distributed as 
\begin{equation}
   \prob{n_p = k} = e^{-\lambda} \frac{\lambda^{k}}{k!},
\end{equation}
for $k=0,1,2, \ldots$.

\subsection{Generation 0}

The first and easiest question is what proportion of each clade will come from generation 0. 

\begin{lemma}
  \label{lem:g0}
Given the process described in \autoref{alg:main}, the expected proportion of surreal forms with generation number 0 is
\begin{equation}
  g_0 =  e^{-\lambda} . \label{eqn:gen_0}  
\end{equation}
\end{lemma}

\begin{proof}
 The only surreal from generation zero is 
 $\ol{0} = \{ \phi \mid \phi \}$, and this is the only surreal with zero parents. So the proportion of a clade from generation 0 (denoted by $g_0$) will be given by the probability of selecting zero parents, \ie
 \(
  g_0 = \prob{n_p = 0 } = e^{-\lambda} . 
 \)
\end{proof}

\subsection{Later generations}

\begin{theorem}
  \label{thm:gk}
  Given the process described in \autoref{alg:main}, and given a weighting $w_k$ and existing clade with generation distribution $g_k$, then the generation distribution of a new clade created from this existing clade will have CDF
  \begin{equation}
     G'_{k+1} = e^{\lambda (Z_k-1)},
  \end{equation}
  where we define weighted PMF $z_k = g_k w_k / Z$, with normalizing constant $Z = \sum_{i=0}^\infty g_i w_i$, and $Z_k = \sum_{i=0}^k z_k$
  for $k \geq 0$. 
\end{theorem}
\begin{proof}
Given clade $\mathcal{C}$ of size $n$, we select individual $x$ with probability proportional to $w\big(g(x)\big) / n$.  If we have a proportion $g_k$ of clade $\mathcal{C}$ from generation $k$, then the probability of a parent $a$ being chosen from generation $k$ is 
\( \prob{ g(a) = k }
  = g_k w_k / Z = z_k.
\)

Given a set of discrete RVs $\{ X_i \}_{i=1}^n$, with Cumulative Distribution Function (CDF) $F_X(\cdot)$, then the CDF of their maximum will be 
$F_X(j)^n.$ 
So the maximum generation of a set of $n_p$ parents $P$ selected as above will have conditional distribution function 
\[ \prob{ \max_{x \in P} g(x) \leq k \; \Bigg| \; |P|=n_p } =Z_k^{n_p},
\]
for $k \geq 0$ and $n_p \geq 0$, where the case $n_p=0$ arises as in Lemma~\ref{lem:g0} because in this case the only possible surreal is $\ol{0}$.

The new clade of surreals $x'$ with $n_p$ sampled parents $P$ will have generation one more than its youngest parent, hence 
\begin{equation}
  \prob{ g(x') \leq k+1 \; \Bigg| \; |P| = n_p }  = Z_k^{n_p}.
    \label{eqn:g_to_z}
\end{equation} 
Summing \autoref{eqn:g_to_z} over the Poisson-distributed parents we see
\begin{equation}
 G'_{k+1}
    =  e^{-\lambda} 
        \sum_{n_p=0}^{\infty}  \frac{(\lambda Z_k)^{n_p}}{n_p!}
   =     e^{\lambda (Z_k-1)}.
\end{equation}
Thus, we have a closed-form expression for the generation distribution of a new clade, given the distribution of the prior clade. \end{proof}

\subsection{The weighting function}

It is easy to compute the formula above numerically, but it is not trivial to use it as we wish. The flexibility of the weighting function makes it hard to  answer questions such as: 
\begin{itemize}

    \item For what weighting functions $w(\cdot)$, if any, does this
      process converge, such that the generation distribution of the
      pre- and post-clade populations are eventually the same? 

    \item Does the initial clade influence the long-term behaviour?


\end{itemize}
In order to answer these questions we restrict our possible weighting
functions as follows: 
for a given generation distribution $g_i$ in the existing clade, we set the weighting function to be $w_i=0$ where $g_i=0$ and where-ever $g_i>0$ we set 
\( w_i = f_i / g_i , \)
for some function $f_i$ independent of $g_i$. That is, we remove the influence of the current prevalence of a particular generation by sampling in inverse proportion to that prevalence. 

Then $Z_k$ doesn't depend on the previous clade (except where it is zero, which we consider below), and (given the maximum generation of the existing population is $g_{max}$) it is given by
\begin{equation}
  Z_k = \frac{\sum_{i=0}^k f_i}{\sum_{i=0}^{g_{max}} f_i},
  \label{eq:z_f}
\end{equation}
for all $k \leq g_{max}.$ 

Here we use the function
\( f_k = \alpha^k, \)
for $\alpha \in (0,1)$. The choice makes some sense -- we suspect that the tail of the generation distributions to be geometric\footnote{We later show that this is the case for this
weighting.} and 
matching this is a simple choice. We leave exploration of other possibilities for future work. Then

\begin{theorem}
  \label{thm:surreal_dist}
  Given the process described in \autoref{alg:main} and given a weighting $w_k = f_i / g_i$ where $g_i$ is the generation distribution of an existing clade and $f_i = \alpha^k$, for $\alpha \in (0,1)$ then the long-term generation distribution of the clades will have PMF 
  \begin{equation}
    \label{eqn:pmf}
   \prob{ g(x) = k }  
      =   \left\{ \begin{array}{ll}  
              e^{-\lambda},  & \mbox{ for } k=0, \\
              e^{-\lambda \alpha^{k}}  - e^{-\lambda \alpha^{k-1}}, & \mbox{ for } k \geq 1. \\
              \end{array}
              \right.
  \end{equation}
\end{theorem}

\begin{proof}
Assume for the moment that $g_k > 0$ for all $k$ in the existing clade, then from \autoref{eq:z_f} and the definition of $f_i$ we get
\begin{equation}
  Z_k  \; = \; (1 - \alpha) \sum_{i=0}^k \alpha^i 
      \; = \; 1  - \alpha^{k+1}. \label{eqn:Z_k}
\end{equation}Now \autoref{eqn:gen_0} and  \autoref{eqn:Z_k} lead to CDF
\begin{equation}
 \prob{ g(x) \leq k  } 
  =   e^{\lambda (Z_{k-1}-1)}
  =   e^{-\lambda \alpha^{k}}. \label{eqn:cdf}
\end{equation}
For $\alpha \in (0,1)$, this function is increasing, non-negative and
converges to 1, making it a valid CDF with PMF \autoref{eqn:pmf}.

The result refers to one iteration (from pre- to post-clade) distribution, assuming that $g_k > 0$ for all $k$, though for a given clade we may have $g_k=0$. The proof of convergence for these cases parallels that of ergodicity in a Markov chain in that we first note that all states are reachable, \ie we can obtain any generation number from any starting state through incremental increases and re-insertions of $\ol{0}$. Moreover, if all parents had odd generation in the pre-clade, then the post-clade would all have even generation. However, the converse is not true because of the insertions of $\ol{0}$. These insertions also remove the possibility of periodicity in the generation-number state. Hence, from any start state any other possible state can be reached, and the number of steps is not periodically constrained.  
\end{proof}

The convergence demonstrated here is a {\em weak} in the sense that we have only shown that a property of the ensemble converges. However, this is sufficient for our purposes, and we will examine other aspects of convergence empirically in later sections. 

This CDF does not appear to be a member of the standard zoo of discrete distributions. We refer to it here
as the {\em surreal distribution} $Su(\lambda, \alpha)$, which is a discrete, univariate, two-parameter distribution with parameters $\lambda > 0$  and $0 < \alpha < 1$. It has support on the non-negative integers, and PMF and CDF are given in \autoref{thm:surreal_dist}. 

\autoref{fig:_gen} shows both an empirical PMF and the predicted PMF from \autoref{eqn:pmf}, along with the geometric approximation described below. We see similar results for a wide range of other parameters. We can see this as retrospective support for our choice of $f_i$, because simple approximations show that the distribution has a geometric tail:
$\prob{g(x) = k} \simeq \lambda \big( 1 - \alpha \big) \alpha^{k-1},$
 as seen in the figure.

\begin{figure}[!t]
  \centering
    \includegraphics[width=\figurewidth]{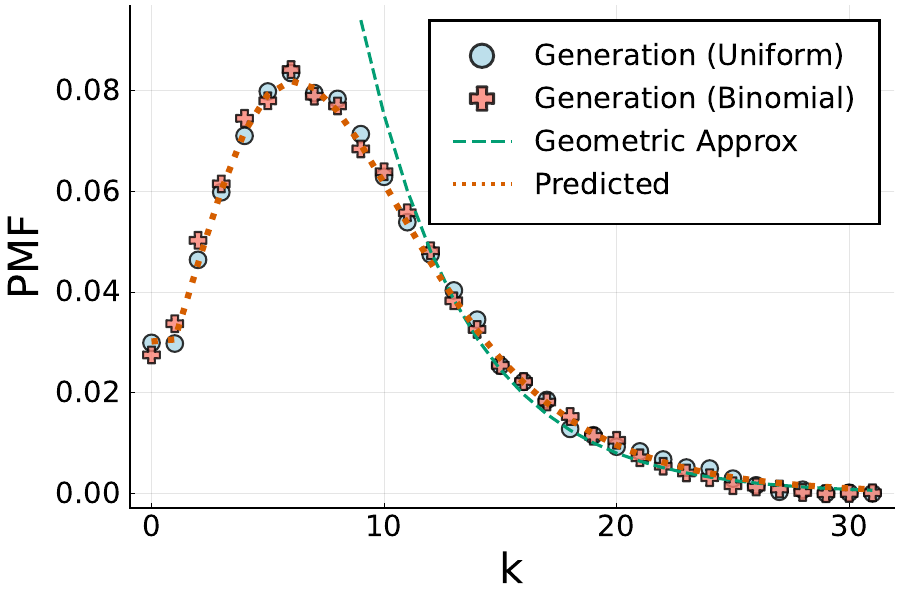}
  \caption{The predicted PMF of the generation distribution showing the
    empirical distributions (derived from 30
    simulations, iterated through 50 clades, with population size $n=500$ and $g^{(0)}_{max} = 1$, $\alpha=0.8, \lambda=3.5$) and the predicted surreal distribution. The geometric
    approximation is also shown (dashed line), and the empirical distributions are shown for both splitting functions (Uniform and Binomial), though there is no significant difference.}
    \label{fig:_gen}  
    \vspace*{-3mm}   
\end{figure} 

\section{Empirical Results}

As yet we have only considered the generation distribution. 
The underlying distribution of surreal forms may converge in a somewhat different manner than this distribution. That is, the population might reach equilibrium after the generation numbers appear to stabilise. We seek to further explore this in the following.

The empirical results shown here are generated by  {\em in silico} experiments. Unless stated, we use 50 iterations to generate a final population, and 30 instances of the process to generate statistics. 

\subsection{Convergence}

The first consideration is convergence. The goal is to examine convergence more generally by investigating other statistics. The surreal forms are DAGs so graph metrics, \eg the number of nodes and edges in the graphs. 

\autoref{fig:_clades} shows how these statistics converge starting from a small initial population of $\{ \ol{0},\ol{-1},\ol{1} \}$, \ie the case where $g^{(0)}_{max} = 1$. We have considered a large set of alternative parameters, and alternative views such as considering the maximum of these distributions, and these two sets are a reasonable representation of the types of behaviour observed. We selected 50 iterations in the following because it was sufficient iterations to see convergence in all examples considered. The population size used in the displayed results ($n=4000$) will be explained in more detail below. We see in the figure that
\begin{itemize}
  \item The statistics converge. Similar results are observed (not shown here) for other statistics, \eg higher order moments and statistics of the values or proportion of integers. 
  \item Convergence is faster for smaller $\lambda$ and $\alpha$, because there is more work to do to get from the small initial population to one with the larger average number of nodes.
  \item The generation distribution converges faster than the graph statistics, as shown by the vertical lines that show iteration at which the value first reaches 99\% of the eventual mean.
\end{itemize}

\begin{figure}[!t]
  \includegraphics[width=\figurewidth]{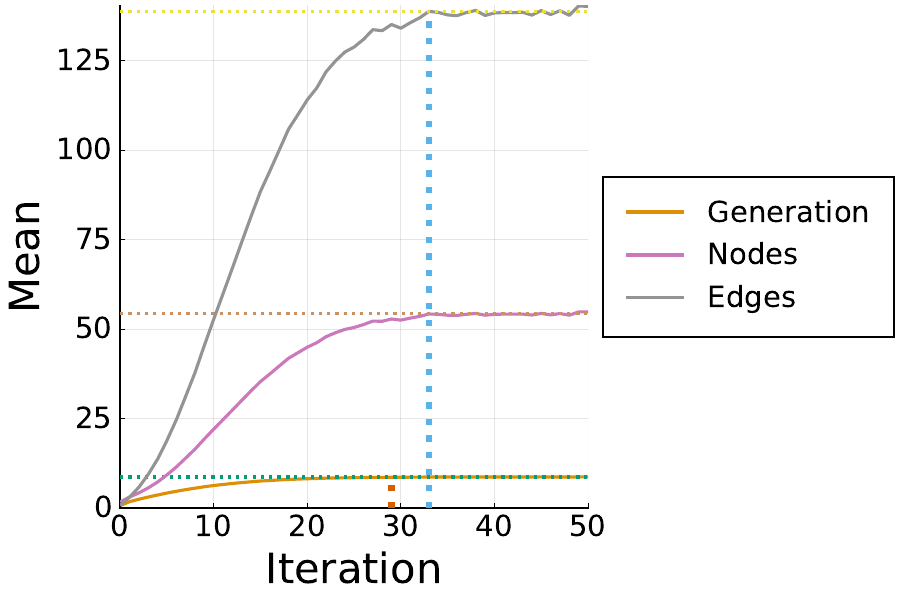}
  \caption{The mean of average generation, and number of nodes and edges for elements of a sequence of clades (for population size $n=4000$ and $g^{(0)}_{max} = 1$, $\alpha=0.8, \lambda=3.5$). Dotted vertical lines show the iteration at which the value first reaches 99\% of the eventual mean.}  
    \label{fig:_clades}  
    \vspace*{-3mm}   
\end{figure} 

Once we believe convergence is happening, the next question is how the parameters of the system affect convergence. We have examined the impact of the initial population, and it is largely inconsequential. It impacts the first few generations but is quickly washed away, as we might hope. 

We will examine the impact of the choice of split distribution in more detail below, but note that (as illustrated in \autoref{fig:_gen}) it has little impact on the convergence process. 

Thus, the main parameter of interest with regard to convergence is the population size $n$. We see little to no impact for smaller $\lambda$ and $\alpha$ values, but for larger values, we can see in \autoref{fig:_pop_size} that population size has a potentially surprising impact. 

The surprising detail is that it has little effect on convergence speed. If we analogise too closely with a GA, we would expect that a larger population might increase convergence at least marginally, but the analogy is flawed.  Convergence is not strongly affected by population size because convergence (here) is not about exploring a space to maximise a fitness function.

On the other hand, the final distribution is impacted. This effect is not easily observable in the generation distribution  \autoref{fig:_pop_size}~(A), but is obvious in the node-size distribution \autoref{fig:_pop_size}~(B). It is easily explainable, however. When $n$ is small, we simply cannot observe the larger, tail cases. Each surreal form in the distribution is built from other forms, so a minimum population size is needed to see the full potential variation. 

\begin{figure*}[!t]
  \centering
  \hfill
  \begin{subfigure}[b]{\figurewidth}
    \centering
    \includegraphics[width=\textwidth]{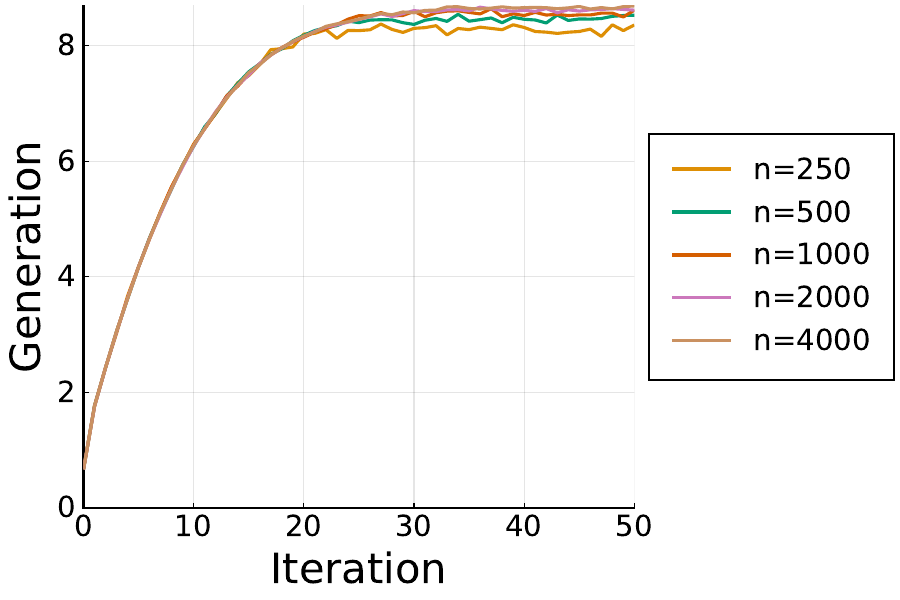}
    \caption{Generation.}
  \end{subfigure}
  \hfill
  \begin{subfigure}[b]{\figurewidth}
    \centering
    \includegraphics[width=\textwidth]{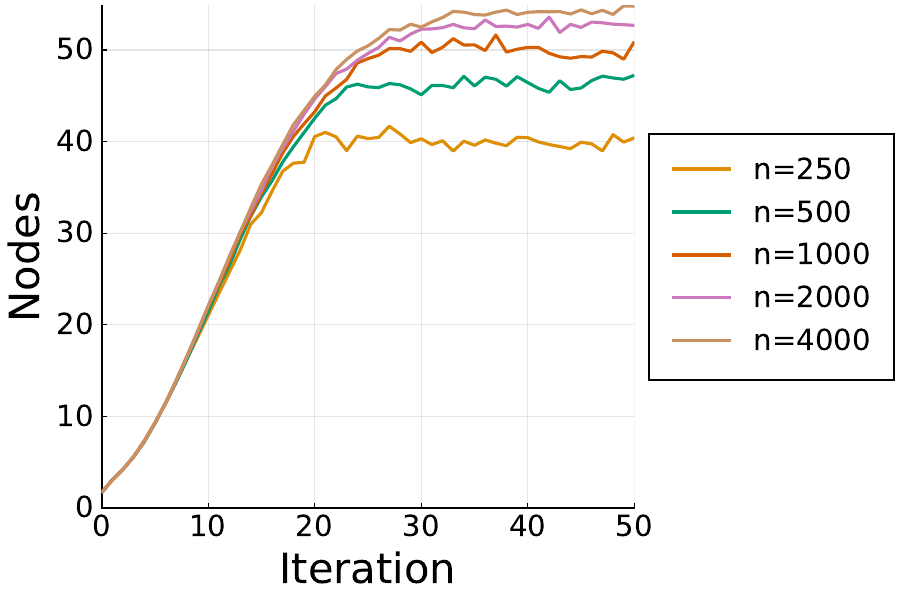}
    \caption{Nodes.}
  \end{subfigure}
  \hspace{2mm}
  \vspace*{-3mm}   
    \caption{Convergence WRT population size $n$ ($\alpha=0.8, \lambda=3.5$) showing node-size converges to different values.}
    \label{fig:_pop_size}   
\end{figure*} 
 
In principle, we should therefore work with an infinite population size but in practice, we find that most variation is seen by $n= 4000$, and the computational cost (which is linear in $n$) vs the marginal improvements in the distribution mean that $n=4000$ is a reasonable compromise. We use this value through the results here (unless specifically stated). 

An additional question is how the time-to-convergence is affected by the other key parameters $\lambda$ and $\alpha$. We measure convergence time by the time until variables such as average generation and node number converge to within 1\% of their final value (as measured from the data). \autoref{fig:converge-time}  shows the number of iterations until convergence. We note that although both $\alpha$ and $\lambda$ impact the convergence time, the impact of $\alpha$ is far larger. We can model the convergence time as 
\( t_{conv} \simeq 1.9 \exp( 3.3 \alpha ).
\)
 
\begin{figure*}[!t]
  \centering 
  \hfill
  \begin{subfigure}[b]{\figurewidth}
    \centering
    \includegraphics[width=\textwidth]{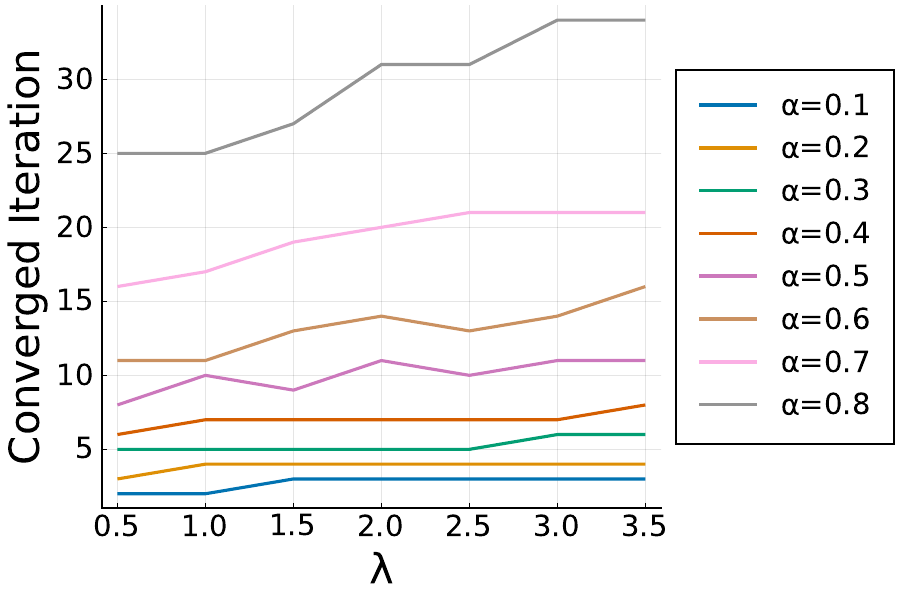}
    \caption{As a function of $\lambda$.}
  \end{subfigure}
  \hfill
  \begin{subfigure}[b]{\figurewidth}
    \centering
    \includegraphics[width=\textwidth]{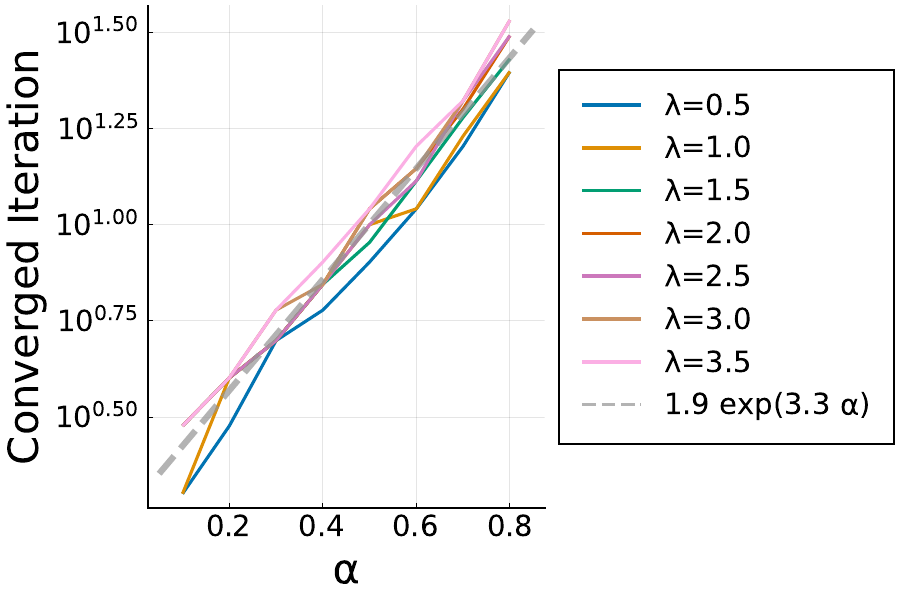}
    \caption{As a function of $\alpha$.}
  \end{subfigure}
  \hspace{2mm}
  \vspace*{-3mm}   
  \caption{Convergence time, \ie until variables converge to within 1\% of their final value ($n=4000$).}
    \label{fig:converge-time}  
\end{figure*}

\subsection{Final ensemble structure}

The evidence supports that convergence is happening, so it is interesting to consider the distributions of the final clade. We already showed the generation distribution in \autoref{fig:_gen}, so we now examine the distributional properties of the DAGs representing the surreal forms. \autoref{fig:_nodes} shows the PMF of the number of nodes. The body is reminiscent of the Poisson distribution (though the details differ). Log-log CCDF plots confirm that the tail is not heavy -- it appears to be roughly geometric, as for the generation distribution, though with different parameters.

We observe similar distributions for a range of parameters and for the two splitting distributions. We also observe a similar pattern for the number of edges in the resulting DAGs. 

One additional concern is the nature of the relationship between variables such as generation, and numbers of nodes and edges. \autoref{fig:_scatter} shows the relationship between these variables for the final clade with $n=4000, m=50, g^{(0)}_{max}=1, \alpha=0.8, \lambda=3.5$ (we have tested and seen similar relationships for a much wider set of parameter choices). \autoref{fig:_scatter} shows one scatter point for each member of the ensemble. It also shows (as crosses) the mean number of nodes for the members of a generation and a quadratic fit (to the raw points, not to the means). Note that the quadratic fit almost exactly matches the mean values for all but the largest few generations. That deviation (for large generations) seems to occur in part because there are far fewer members in these generations, and in part because the population sample is limited ($n=4000$) and the impact of this is more keenly felt in the (large generation) tail. 

The quadratic relationship is explainable through the following argument: each surreal form is described by a DAG with exactly one source (the root, or a node with no parents, which here is the canonical zero $\ol{0}$), and one sink (a node with no children). For instance see \autoref{fig:x0}(B). So if we plot each node on the surreal form such that its height equals its generation number, we will see a graph that is narrow at the top and bottom, and wide in the middle. More precisely, if we start at the top of the DAG (at the sink), the graph will initially widen as we go down because each node has several parents. However, as we approach the lower reaches of the graph (near the source) there are limited options. There are a smaller number of potential parents with small birthdays, so at some point the graph starts to narrow, eventually to just one node. So in essence, for any given surreal form we observe a plot that might be somewhat diamond-shaped, with the height of the diamond determined by the generation of the surreal. The area of a diamond is quadratic in its height, and hence the relationship. Of course, this argument is rough, and there is a high degree of variability between individual surreal forms (we see that variability in \autoref{fig:_scatter}).

We see (plot not shown) an even stronger linear relationship between the number of nodes in a surreal form and the number of edges. This is hardly surprising as each node (except the last) is the parent of some other node in the graph. 
 
\begin{figure}[!thp]
  \centering
    \includegraphics[width=\figurewidth]{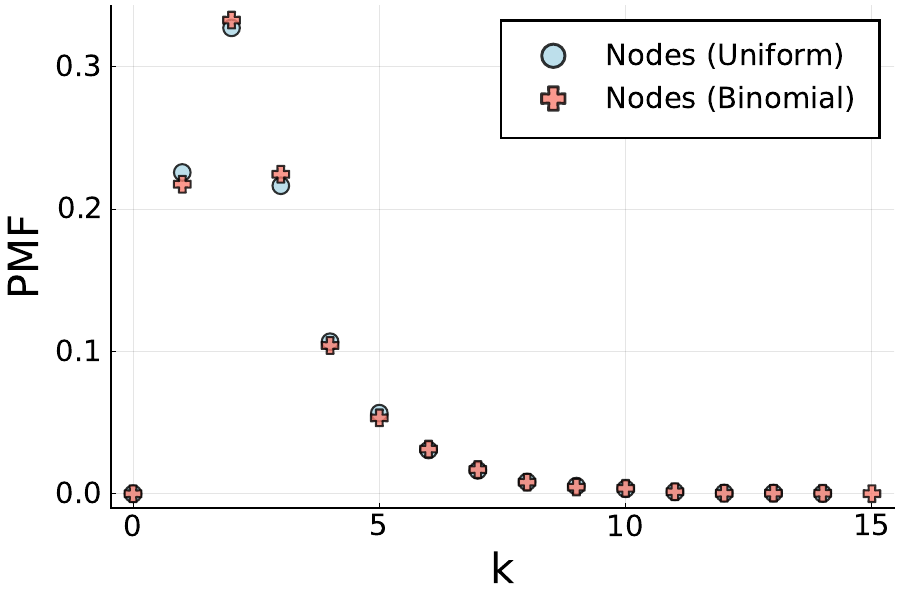}
  \vspace{-3mm}
  \caption{The distribution of the number of nodes in DAG graphs ($n=4000, m=50, g^{(0)}_{max} =1$, $\alpha=0.4, \lambda=1.5$).}
    \label{fig:_nodes}  
    \vspace*{-5mm}   
\end{figure}

\begin{figure}[!t]
  \centering
  \includegraphics[width=\figurewidth]{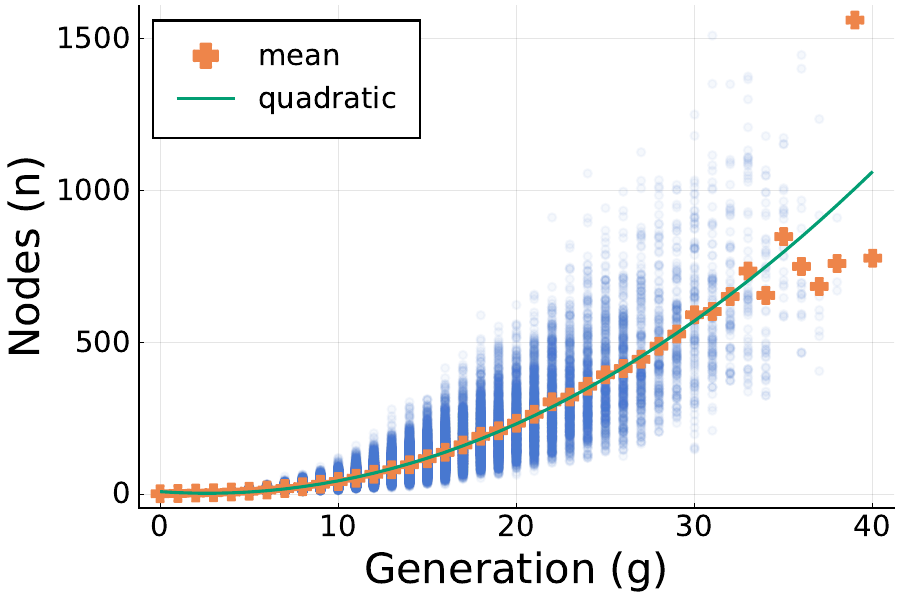} 

  \caption{The relationship between generation and numbers of nodes in the final clade ($n=4000, m=50, g^{(0)}_{max}=1, \alpha=0.8, \lambda=3.5$). Each point is one member of the set of generated clades and crosses show the mean for each generation. The solid line shows a fit (to the raw data not the means).
  }
    \label{fig:_scatter}  
    \vspace*{-3mm}   
\end{figure}

\section{The splitting function}

The previous aspects of the ensemble (generation, DAG size, convergence time) considered above are not strongly influenced by the splitting function. In this section, we consider aspects of the ensemble that are impacted strongly by the choice of splitting function. 

The most obvious impact of the splitting function is on the value of a surreal number. Conway's simplicity theorem~\cite{conway2000numbers}[p.23] shows that if there is a $v(x)$ that computes the value of a surreal form $x$, then for finite surreals
\begin{equation}
  \label{eqn:values}
  \max_{x_L \in X_L} v(x_L) < v(x) < \min_{x_R \in X_R} v(x_R).
\end{equation} 
Conway's theorem further implies that the value will be the one corresponding to the simplest possible surreal that satisfies \autoref{eqn:values}. The impact here is that if the split-point is more often in the middle of the parent list $P(x)$ then that will have a strong influence on the value of the resulting surreal.

The value of a surreal number is important for some algorithms. Here we consider the prevalence of integers in the ensemble because this is relevant for algorithms such as {\tt floor}. We can derive an approximation for the proportion of integer surreals that will be generated by \autoref{alg:main} as follows, making repeated use of Simons~\cite[p.11]{simons17:_meet_the_surreal_number} {\em Extra Option Theorem} so we repeat it here:

\begin{theorem}[Simon's Extra Option Theorem]
  \label{thm:extra_option}
  If $x == \{X_L \mid X_R \}$ is a surreal form, and $l$ and $r$ are surreal numbers such that $l < x < r$, then
      \( \{l, X_L \mid X_R \} \equiv x \equiv \{X_L \mid r, X_R \}.
      \)
\end{theorem}

\noindent We can use this to approximate the proportion of integers as follows:

\begin{theorem}
  \label{thm:integers}
  Any short surreal form with an empty left or right set will be an integer. Further, its value will be the integer of smallest magnitude that satisfies \autoref{eqn:values}.
\end{theorem} 

\begin{proof}
  Assume a surreal $x$ with empty right set, and hence non-empty left set. 
   Now take $x_m = \max X_L$ and create surreal $x' \defequals \{ x_m | \phi \}$, then by \autoref{thm:extra_option} $x' \equiv x$. 
  
  If $x_m < 0$, then the simplest surreal that satisfies \autoref{eqn:values} for $x'$ (and hence $x$) is $\ol{0}$, and hence the value of $x$ is the integer 0. 
  
  Also note that if $x_m (\geq 0)$ is an integer then by definition $x'$ is the canonical integer $x_m + 1$. 

  Hence, from now assume $x_m > 0$ is not an integer. Create a new surreal $x'' \defequals \big\{ x_m, \lfloor x_m \rfloor \big| \phi \big\}$, where we insert the largest integer smaller than $x_m$, \ie $\lfloor x_m \rfloor$ into the left set noting that by \autoref{thm:extra_option} we have $x'' \equiv x' \equiv x$. 

  We can further define surreal $x''' \defequals \big\{ \lfloor x_m \rfloor \big| \phi \big\}$ and note that by \autoref{thm:extra_option} and \autoref{eqn:values} we have $x''' \equiv x'' \equiv x' \equiv x$. Now $\lfloor x_m \rfloor$ is, by definition, an integer and hence $x'''$ is the canonical integer $\lfloor x_m \rfloor + 1$. Hence, $x$ is an integer. What's more, it is the smallest integer that is larger than $x_m$ and hence satisfies \autoref{eqn:values}.

  Empty left sets follow {\em mutatis mutandi} by symmetry.
\end{proof}

\autoref{thm:integers} implies that whenever the split point $s=0$ or $s=n_p$,
then the left or right sets, respectively, will be empty, and hence the result is an integer. Hence, there is a lower bound on the probability of an integer for uniformly distributed split points of 
\[ \prob{x \in \Z | n_p } \geq 2 / (n_p+1), \]
for $n_p \geq 2$ (and it is 1 otherwise). 

Given a Poisson number of parents, we obtain: 
\begin{eqnarray}
  \prob{x \in \Z } 
         & = & \sum_{n=0}^\infty \prob{x \in \Z | n_p } p(|P| = n_p) \nonumber \\
         & \geq & e^{-\lambda} + \sum_{n=1}^\infty \frac{2}{n+1} \frac{e^{-\lambda} \lambda^n}{n!} \nonumber \\
         & = & \frac{2}{\lambda} \big( 1 - e^{-\lambda} \big) - e^{-\lambda}.
\end{eqnarray}

The bound is rough but helpful because it also points out a useful approach to reduce the number of integers --- we need a splitting function with a lower probability that  $s=0$ or $s=n_p$. Here we trial the Binomial distributions.

\autoref{fig:_denominator} shows the empirical results (note that dyadic rationals $n / 2^k$, have denominator-exponent $k$, which we show in the plot --- the integer surreal numbers are those for which the denominator-exponent is $k=0$).

\begin{figure}[!t]
    \includegraphics[width=\figurewidth]{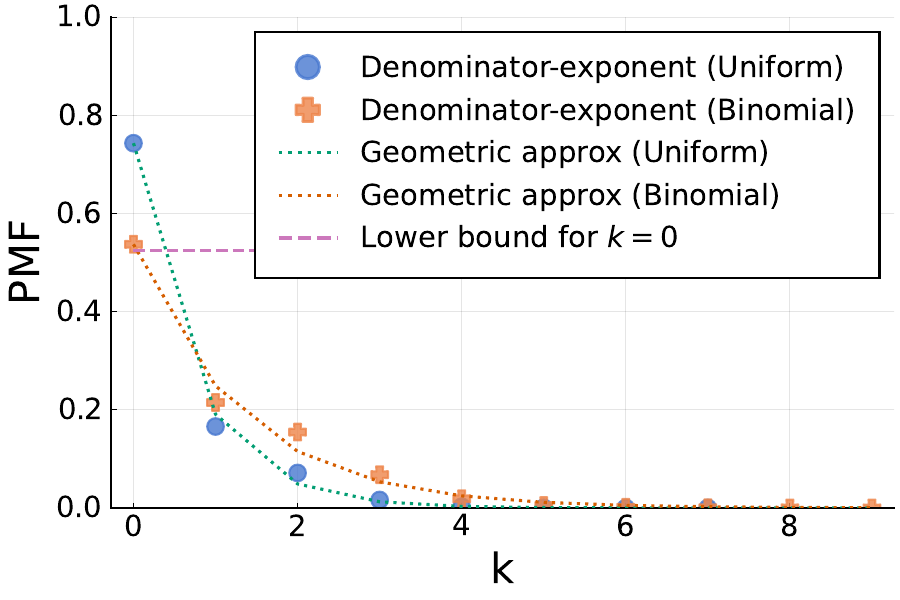}
  \caption{The denominator-exponent in the dyadic values $n/2^k$ (derived from 30 simulations, iterated through 50 clades, with population size $n=4000$ and $g^{(0)}_{max} = 1$, $\alpha=0.8, \lambda=3.5$). Note that the majority of surreals generated are integers, \ie $k=0$. However, this proportion decreases when the Binomial splitting function is preferred. Also shown (dashed line) is our lower bound integer estimate. We can see that the Binomial distribution spitting function approaches the bound for larger $\lambda$. Also shown (dotted) are geometric approximations to the distributions.}
    \label{fig:_denominator}  
    \vspace*{-3mm}    
\end{figure}

We can see a high proportion of integers, $k=0$, in both cases, but a reduction in the number of integers for the Binomial splitting function. The proportion is decreased to near the lower bound. Resampling could be used to further reduce the prevalence if needed.

\section{Conclusion and Future Work}

This paper presents an evolutionary algorithm that creates an ensemble of random surreal number forms with controlled complexity in order to create benchmark datasets. We are already using the results to test new algorithms, with results to be reported soon.

Apart from using the results in testing, there are a number of avenues for future investigation. For instance, there are a number of ways to generalize or extend the synthesis process, for instance by adapting additional ideas from the analogy of evolutionary algorithms. Moreover, we believe that there are likely to be many other areas where large, complex datasets are needed for benchmarking, and evolutionary computing can create such sets.

\vspace{-3mm} 
\section*{}  
{
\bibliographystyle{cpclike} 
\bibliography{surreals,graph}  

\begin{thebibliography}{}

\bibitem{adeniyi16:_improv_genet_algor_based_test}
Adeniyi, A.~M. and Olalekan, A.~S. (2016) An improved genetic algorithm-based
  test coverage analysis for graphical user interface software. {\em American
  Journal of Software Engineering and Application} {\bf 5} 7--14.
\newblock doi: \url{10.11648/j.ajsea.20160502}.

\bibitem{conway2000numbers}
Conway, J. (2001) {\em On Numbers and Games}.
\newblock A K Peters/CRC Press, Wellesley, USA.

\bibitem{giuffre13:_harnes}
Giuffr\`{e}, M. and Shung, D. (1013) Harnessing the power of synthetic data in
  healthcare: innovation, application, and privacy. {\em npj Digital Medecine}
  {\bf 6}.
\newblock DOI: \url{https://doi.org/10.1038/s41746-023-00927-3}.

\bibitem{grimm12:_introd_surreal_number}
Grimm, G. (2012) An introduction to surreal numbers.
\newblock
  \url{https://www.whitman.edu/Documents/Academics/Mathematics/Grimm.pdf},
  (accessed Sept 25th, 2018).

\bibitem{jones96:_autom}
Jones, B., Sthamer, H.-H., and Eyres, D. (1996) Automatic structural testing
  using genetic algorithms. {\em Software Engineering Journal} {\bf 11}
  299--306.

\bibitem{keddie94:_ordin_operat_surreal_number}
Keddie, P. (1994) Ordinal operations on surreal numbers. {\em Bulletin of the
  London Mathematical Society} {\bf 26} 531--538.

\bibitem{knuth1974surreal}
Knuth, D. (1974) {\em Surreal Numbers: How Two Ex-students Turned on to Pure
  Mathematics and Found Total Happiness: a Mathematical Novelette}.
\newblock Addison-Wesley Publishing Company.

\bibitem{pargas99:_test}
Pargas, R.~P., Harrold, M.~J., and Peck, R.~R. (1999) Test-data generation
  using genetic algorithms. {\em Software Testing, Verification and
  Reliability} {\bf 9} 263--282.

\bibitem{roughan2008}
Ringberg, H., Roughan, M., and Rexford, J. (2008) The need for simulation in
  evaluating anomaly detectors. {\em Computer Communication Review} {\bf 38}
  55--59.

\bibitem{ROUGHAN2019293}
Roughan, M. (2019) Practically surreal: Surreal arithmetic in {J}ulia. {\em
  SoftwareX} {\bf 9} 293--298.

\bibitem{ROUGHAN_s2023}
Roughan, M. (2023) Surreal birthdays and their arithmetic. {\em Mathematics
  Magazine} {\bf 96} 329--343.

\bibitem{schleicher2005introduction}
Schleicher, D. and Stoll, M. (2005) An introduction to {C}onway's games and
  numbers.
\newblock arXiv, math/0410026.

\bibitem{simons17:_meet_the_surreal_number}
Simons, J. (2017) Meet the surreal numbers.
\newblock \raggedright
  \url{https://www.m-a.org.uk/resources/.../4H-Jim-Simons-Meet-the-surreal-numbers.pdf}.

\bibitem{tondering13:_surreal_number_introd}
T{\o}ndering, C. (2013) Surreal numbers -- an introduction.
\newblock Version 1.6, \url{https://www.tondering.dk/download/sur16.pdf},
  (accessed Sept 25th, 2018).

\bibitem{b.m.waxman88:_graph}
Waxman, B. (1988) Routing of multipoint connections. {\em IEEE J. Select. Areas
  Commun.} {\bf 6} 1617--1622.

\bibitem{wilde20:_evolut}
Wilde, H., Knight, V., and Gillard, J. (2020) Evolutionary dataset
  optimisation: learning algorithm quality through evolution. {\em Applied
  Intelligenc} {\bf 50} 1172--1191.
\newblock doi: \url{10.1007/s10489-019-01592-4}.

\end{thebibliography}
}\par\leavevmode

\label{lastpage}
\end{document}